\PassOptionsToPackage{pdfpagelabels=false}{hyperref} 
\documentclass[english,a4paper,11pt]{article}
\usepackage{amsmath, amsxtra, amsfonts, amssymb, amstext}
\usepackage{amsthm}
\usepackage{booktabs}
\usepackage{fullpage}
\usepackage{nicefrac}
\usepackage{xspace}
\usepackage[noadjust]{cite}
\usepackage{url}
\usepackage{graphics,color}
\usepackage[colorlinks]{hyperref}
\definecolor{linkblue}{rgb}{0.1,0.1,0.8}
\hypersetup{colorlinks=true,linkcolor=linkblue,filecolor=linkblue,urlcolor=linkblue,citecolor=linkblue}
\usepackage[algo2e,ruled,vlined,linesnumbered]{algorithm2e}
\newcommand{\assign}{\leftarrow}

\newtheorem{theorem}{Theorem}
\newtheorem{lemma}[theorem]{Lemma}

\newtheorem{corollary}[theorem]{Corollary}
\newtheorem{definition}[theorem]{Definition}

\newcommand{\ignore}[1]{}

\allowdisplaybreaks[1]

\newcommand{\N}{\mathbb{N}}
\newcommand{\R}{\mathbb{R}}

\renewcommand{\epsilon}{\varepsilon}
\newcommand{\eps}{\varepsilon}

\newcommand{\p}{\vec{p}}
\newcommand{\q}{\vec{q}}

\DeclareMathOperator{\E}{E}

\newcommand{\onemax}{\textsc{OneMax}\xspace}

\newcommand{\OneMax}{\textsc{OneMax}\xspace}
\newcommand{\OM}{\textsc{Om}\xspace}
\newcommand{\LeadingOnes}{\textsc{LeadingOnes}\xspace}
\newcommand{\leadingones}{\textsc{LeadingOnes}\xspace}
\newcommand{\LO}{\textsc{Lo}\xspace}

\newcommand{\oea}{$(1 + 1)$~EA\xspace}

\newcommand{\OEAi}{$(1 + 1)$~EA$_i$\xspace}
\newcommand{\OEAp}{$(1 + 1)$~EA$_{\p}$\xspace}
\newcommand{\OEAq}{$(1 + 1)$~EA$_{\p}$\xspace}
\newcommand{\OEAQ}{$(1 + 1)$~EA$_{Q}$\xspace}

\newcommand{\TrunkGeo}{\mathrm{TrunkGeo}}
\newcommand{\Geo}{\mathrm{Geo}}

\begin{document}
    
\title{Solving Problems with Unknown Solution Length at (Almost) No Extra Cost}

\author{
Benjamin Doerr, \'Ecole Polytechnique, Paris-Saclay, France\\
Carola Doerr, CNRS and Sorbonne Universit\'es, UPMC Univ Paris 06, Paris, France\\
Timo K{\"o}tzing, Friedrich-Schiller-Universit{\"a}t, Jena, Germany
}
\date{This is a preliminary version of a paper that is to appear at the Genetic and Evolutionary Computation Conference (GECCO 2015).}

\maketitle

{\sloppy
\begin{abstract}
Most research in the theory of evolutionary computation assumes that the problem at hand has a fixed problem size. This assumption does not always apply to real-world optimization challenges, where the length of an optimal solution may be unknown a priori.
 
Following up on previous work of Cathabard, Lehre, and Yao [FOGA 2011] we analyze variants of the (1+1) evolutionary algorithm for problems with unknown solution length. For their setting, in which the solution length is sampled from a geometric distribution, we provide mutation rates that yield an expected optimization time that is of the same order as that of the (1+1) EA knowing the solution length.

We then show that almost the same run times can be achieved even if \emph{no} a priori information on the solution length is available.

Finally, we provide mutation rates suitable for settings in which neither the solution length nor the positions of the relevant bits are known. Again we obtain almost optimal run times for the \textsc{OneMax} and \textsc{LeadingOnes} test functions, thus solving an open problem from Cathabard et al. 
\end{abstract} 


\section{Introduction}%
\label{sec:introduction}
While the theory for evolutionary algorithms (EAs) in static problem settings is well developed~\cite{AugerD11,Jansen13,NeumannW10}, a topic that is not so well studied in the theory of EA literature is the performance of EAs in uncertain environments. Uncertainty can have many faces, for example with respect to function evaluations, the variation operators, or the dynamics of the fitness function. Understanding how evolutionary search algorithms can tackle such uncertain environments is an emerging research topic; see~\cite{Bia-Dor-Gam-Gut:j:09} for a survey on examples in combinatorial optimization, but also~\cite{JinBranke:j:05:robustSurvey} for an excellent survey also discussing different sources of uncertainty. 

In this work we study what evolutionary algorithms can achieve in the presence of uncertainty with respect to the solution length. Quite surprisingly, we show that already some variants of the simplest evolutionary algorithm, the \oea, can be very efficient for such problems.

\subsection{Previous Work}
Our work builds on previous work of Cathabard, Lehre, and Yao~\cite{PKLFoga11}, who were the first to consider, from a theoretical point of view, evolutionary algorithms in environments with unknown solution lengths. Cathabard et al.\ assume that the solution length is sampled from a fixed and known distribution $D$ with finite support. More precisely, they assume that the solution length $n$ is sampled from a truncated version of the geometric distribution, in which the probability mass for values greater than some threshold $N$ is shifted to the event that $n=N$. In this situation, the algorithm designer has access to both the upper bound $N$ for the solution length and the success probability $q$ of the distribution.

Cathabard et al.\ analyze a variant of the \oea in which each bit is flipped with probability $1/N$ and they also study a variant with non-uniform bit-flip probabilities. In the latter, the $i$-th bit is flipped independently of all other bits with probability $1/(i+1)$. They show that these variants have polynomial expected run times on \onemax and \leadingones function, where the expectation is taken with respect to the solution length and the random decisions of the algorithm. An overview of the precise bounds obtained in~\cite{PKLFoga11} is given in Table~\ref{tab:PKLsetting}.

\subsection{Our Results}

We extend the work of Cathabard et al.\ in several ways. In a first step (Section~\ref{sec:RandomSolutionLength}) we show that the regarded mutation probabilities are sub-optimal. Making use of the concentration of the (truncated) geometric distribution, we design bit flip probabilities that yield significantly smaller expected run times (for both the \onemax and the \leadingones function). We complement this finding by a lower bound that shows the optimality of our result. This proves that no mutation probabilities can yield a performance that is better by more than a constant factor than our suggested ones.

While in the setting of Cathabard et al.\ we are in the convenient situation that we have full knowledge of the distribution $D$ from which the solution length is sampled, one is sometimes faced with problems for which this knowledge is not readily available. 
We therefore study in Section~\ref{sec:ArbitrarySolutionLength} what can be done without any a priori knowledge about the solution length. In this situation we require that the algorithm designer chooses bit flip probabilities $(p_i)_{i\in \N}$ such that, regardless of the solution length $n$, the expected performance of the \oea with bit flip probabilities $(p_1, \ldots, p_n)$ is as small as possible. It is not obvious that this can be done in polynomial time. In fact, for both algorithms studied by Cathabard et al.\ as well as for any uniform choice of the bit flip probabilities, the expected run time on this problem is exponential in $n$ (cf. Theorems~\ref{thm:ourNegOneMax} and~\ref{thm:ourNegLeadingOnes}). 

We show (Theorems~\ref{thm:ourReformOneMax} and~\ref{thm:ourReformLeadingOnes}) that not only can we tackle this problem with non-uniform bit flip probabilities, but, quite surprisingly, this can be even done in a way that yields almost optimal run times. Indeed, our results are only a $\log^{1+\eps}n$ factor worse than the best possible $\Theta(n \log n)$ and $\Theta(n^2)$ run time bounds for \onemax and \leadingones, respectively. This factor can be made even smaller as we shall comment at the end of Section~\ref{sec:ArbitraryLengthNonUniformBitFlipProb}.

Finally, we provide in Section~\ref{sec:ArbitraryLengthRandomBitFlipProb} a second way to deal with unknown solution lengths. We provide an alternative variant of the \oea in which the bit flip probabilities are chosen according to some (fixed) distribution at the beginning of each iteration. For suitably chosen distributions $Q$, the expected run times of the respective \OEAQ on \onemax and \leadingones are of the same asymptotic order as those of the previously suggested solution with non-uniform bit flip probabilities. In particular, they are, simultaneously for all possible solution lengths $n$, almost of the same order as the expected run time of a best possible \oea knowing the solution length.

This second approach has an advantage over the non-uniform bit flip probabilities in that it effectively ignores bits that do not contribute anything to the fitness function (\emph{irrelevant bits}). Thus, even if only $n$ bits at unknown positions have an influence on the fitness function, the same run time bounds apply. In contrast, all previously suggested solutions require that the $n$ relevant bits are the leftmost ones. This also answers a question posed by Cathabard et al.~\cite[Section~6]{PKLFoga11}.

Our run time results are summarized in Tables~\ref{table:overviewResults} and~\ref{tab:PKLsetting}. 

\begin{table*}%
\begin{center}
\renewcommand{\arraystretch}{2}
\begin{tabular}{l|c||cl|cl}
		 \textbf{setting} & \textbf{bit flip prob.} & \multicolumn{2}{c|}{\textbf{\OneMax}} & \multicolumn{2}{c}{\textbf{\LeadingOnes}} \\
\hline
\parbox{3cm}{Random Length $\sim \Geo(q)$}  & unif.\ and fixed & $O(q^{-1} \log q^{-1})$ & Thm.~\ref{thm:OurRandomLengthOneMax} &  $O(q^{-2})$ & Thm.~\ref{thm:OurRandomLengthLeadingOnes}\\
\hline
		& unif.\ and fixed & $2^{\Omega(n)}$ & Thm.~\ref{thm:ourNegOneMax} &  $2^{\Omega(n)}$& Thm.~\ref{thm:ourNegLeadingOnes} \\
\cline{2-6}
\parbox{3cm}{Adversarial Length}  & fixed & $O(n\log^{2+\eps} n)$ & 
Cor.~\ref{cor:nonUniformBitFlipProb} & $O(n^2\log^{1+\eps} n)$ & 
Cor.~\ref{cor:nonUniformBitFlipProb}\\
\cline{2-6}
		& unif.\ and rand.\ & $O(n\log^{2+\eps} n)$ & 
		Cor.~\ref{cor:RandomBitFlipProbability} & $O(n^2\log^{1+\eps} n)$ & 
		Cor.~\ref{cor:RandomBitFlipProbability}\\
\hline
\end{tabular}
\end{center}
\caption{Overview of Results for $1/N<q<1/2$ and $\eps>0$.}
\label{table:overviewResults}
\end{table*}

\section{Algorithms and Problems}

In this section we define the algorithms and problems considered in this paper. For any problem size $n$, fitness function $f: \{0,1\}^n \rightarrow \R$, and vector $\p=(p_1,\ldots,p_n)$ of bit flip probabilities $0\leq p_i \leq 1$, we consider the \OEAp, as given by Algorithm~\ref{alg:PKLmodel}.

\begin{algorithm2e}%
	\textbf{Initialization:} 
	Sample $x \in \{0,1\}^n$ uniformly at random and query $f(x)$\;
  \textbf{Optimization:}
	\For{$t=1,2,3,\ldots$}{
		\For{$i=1, \ldots, n$}
			{\label{line:mut}With probability $p_i$ set $y_i\assign 1-x_i$ and set $y_i \assign x_i$ otherwise\;}
		Query $f(y)$\;
		\lIf{$f(y)\geq f(x)$}{$x \assign y$\;}	
}
\caption{The \OEAp for $\p=(p_1,\ldots,p_n)$ optimizing a pseudo-Boolean function $f:\{0,1\}^n \rightarrow \R$.}
\label{alg:PKLmodel}
\end{algorithm2e}
 
The \OEAp samples an initial search point from $\{0,1\}^n$ uniformly at random. It then proceeds in rounds, each of which consists of a mutation and a selection step. Throughout the whole optimization process the \OEAp maintains a population size of one, and the individual in this population is always a best-so-far solution.
In the \emph{mutation step} of the \OEAp the current-best solution $x$ is \emph{mutated} by flipping the bit in position $i$ with probability $p_i$, $1 \leq i \leq n$. The fitness of the resulting search point $y$ is evaluated and in the \emph{selection step} the parent $x$ is replaced by its offspring $y$ if and only if the fitness of $y$ is at least as good as the one of $x$. Since we consider maximization problems here, this is the case if $f(y) \geq f(x)$. Since we are interested in expected \emph{run times}, i.e., the expected number of rounds it takes until the \OEAp evaluates for the first time a solution of maximal fitness, we do not specify a termination criterion. 
It is not difficult to see that the \OEAp indeed generalized the standard \oea. In fact, we obtain the \oea from the \OEAp if we set $p_i=1/n$ for all $i \in [n]:=\{1, \ldots, n\}$. We call such mutation vectors with $p_i=p_j$ for all $i,j$ \emph{uniform} mutation rates, while we speak of \emph{non-uniform} mutation rates if $p_i\neq p_j$ for at least one pair $(i,j)$.

The two test functions we consider in this work are \OneMax and \leadingones. For a given problem size $n$, they are defined as
\begin{align*}
\OM_n
& :=\onemax_n(x)=\sum_{i=1}^n{x_i}, \text{ and}\\
\LO_n
& :=\leadingones_n(x)\\
& =\max\{i \in [0..n] \mid \forall j \leq i: x_j=1\},
\end{align*}
where $[0..n]:=\{0\} \cup [n]$. That is, the \onemax function counts the number of ones in a bit string, while the \leadingones function counts the number of initial ones. While these two functions are certainly easy to optimize without evolutionary algorithms, the \OEAp performs exactly the same on all generalized \onemax and \leadingones functions, which are obtained from the functions above through an XOR of an arbitrary and unknown bit string $z \in \{0,1\}^n$. Understanding how an evolutionary algorithm behaves on these two functions is an important indicator for how it manages to cope with the easier parts of more complex optimization problems. \onemax and \leadingones functions are for this reason the two best-studied problems in the theory of evolutionary computation literature.

If a distribution $D$ is known from which the solution length is sampled we consider the expected run time of the 
\OEAp on $\onemax_D$ and $\leadingones_D$, respectively, which are the problems $\OM_n$ resp. $\LO_n$ with random solution length $n \sim D$. Note here that the expectation is thus taken both with respect to the random solution length and with respect to the random samples of the algorithm.


\section{Random Solution Length}
\label{sec:RandomSolutionLength}

We first consider the setting that has been introduced by Cathabard, Lehre, and Yao~\cite{PKLFoga11}. 
After a short presentation of the model in Section~\ref{sec:PKLsetting}, a general lower bound for this problem (Section~\ref{sec:PKLlowerbound}), and the results of~\cite{PKLFoga11} in Section~\ref{sec:knownResults}, 
we show that the bounds in~\cite{PKLFoga11} can be improved by using different (uniform) mutation rates (Section~\ref{sec:PKLour}).

Table~\ref{tab:PKLsetting} summarizes the previously known bounds and our contributions for the setting regarded in this section.

\begin{table}%
\begin{center}
\renewcommand{\arraystretch}{1.5}
\begin{tabular}{l|c|c|c|c|c}
 & 
 & \multicolumn{2}{|c|}{\textbf{Results from~\cite{PKLFoga11}}}& \multicolumn{2}{|c}{\textbf{Thms.~\ref{thm:OurRandomLengthOneMax} (\textsc{OM})}, \textbf{\ref{thm:OurRandomLengthLeadingOnes} (\textsc{LO})}} \\
\textbf{problem}& \textbf{Cor.~\ref{cor:lowerBoundOnTrunkated}}& $p_i=1/N$ & $p_i=1/(i+1)$ & $p_i=q/2$ & $p_i=q$\\
\hline
$\onemax_D$ & $\Omega\left(\frac{1}{q} \log \frac{1}{q}\right)$ & $\Theta\left(N \log \frac{1}{q}\right)$ & $O\left(\frac{1}{q^2} \log N\right)$ & $\Theta\left(\frac{1}{q} \log \frac{1}{q}\right)$ & $\Theta\left(N \log N\right)$ \\
$\leadingones_D$ & $\Omega\left(\frac{1}{q^2}\right)$& $\Theta\left(\frac{N}{q}\right)$ & $\Theta\left(\frac{1}{q^3}\right)$ & $\Theta\left(\frac{1}{q^2}\right)$ & $\Theta\left(\frac{N}{q}\right)$\\
\hline
\end{tabular}
\end{center}
\caption{Expected run times of the \OEAp with $\p=(p_i)_{i=1}^N$ for $D=\TrunkGeo(N,q)$ and $1/N \leq q \leq 1/2$}
\label{tab:PKLsetting}
\end{table}

\subsection{The Model}
\label{sec:PKLsetting}

Cathabard et al.~\cite{PKLFoga11} consider the following model. 
The algorithm designer knows the distribution $D$ from which the unknown solution length is drawn; only distributions with finite support are considered, so the algorithm designer knows an upper bound $N$ on the actual solution length $n$. 
He also knows the class of functions from which the optimization problem is taken (for example \OneMax or \LeadingOnes). 

Based on this knowledge, the algorithm designer chooses a vector $(p_1,\ldots, p_N)$ of bit flip probabilities indicating with which probability a bit is flipped in each round. In this work we also regard a slightly more general model in which the distributions over $\N$ may possibly have infinite support; the algorithm designer then chooses an infinite sequence of bit flip probabilities $(p_1, p_2,\ldots)= (p_i)_{i \in \N}$. 
After this choice of bit flip probabilities, the actual solution length $n$ is sampled from the given distribution $D$. Then the \OEAp (Algorithm~\ref{alg:PKLmodel}) is run with mutation probabilities $\p=(p_1, \ldots, p_n)$ on the given problem with the given problem length.

Cathabard et al.~\cite{PKLFoga11} consider as distribution $D$ the following \emph{truncated geometric distribution}, based on a geometric distribution where the probability mass for values greater than $n$ are moved to $n$.

\begin{definition}[{\cite{PKLFoga11}}]
The \emph{truncated geometric distribution} $\TrunkGeo(N,q)$ with truncation parameter $N$ and success probability $q \in (0,1/N]$ satisfies, 
for all $n \in \N$, that the probability of $\TrunkGeo(N,q)=n$ is
\begin{align*}
\begin{cases}
q(1-q)^{n-1} &\text{ if } 1 \leq n \leq N-1,\\
(1-q)^{n-1} &\text{ if } n=N,\\
0 &\text{ otherwise.}
\end{cases}
\end{align*} 
\end{definition}
Note that the truncated geometric distribution recovers the geometric distribution $\Geo(q)$ for $N = \infty$. 

It is well known, respectively can be found in~\cite[Proposition 1]{PKLFoga11}, that for $X=\Geo(q)$ and $Y=\TrunkGeo(N,q)$ with $q \geq 1 / N$
\begin{align}
\label{eq:expectedValue}
\E[X]=q^{-1} \text{ and } \E[Y]=\Theta(q^{-1}).
\end{align}
Note that we trivially have $\E[Y] \leq \E[X]$.

\subsection{A General Lower Bound}
\label{sec:PKLlowerbound}

What is a good lower bound for the expected run time of \emph{any} \OEAp on \OneMax or \LeadingOnes when the length is sampled from some given distribution $D$ on $\N$? 
If the algorithm designer would know the true length $n$ before he has to decide upon the mutation probabilities $(p_1, \ldots, p_n)$, then the optimal bit flip probability for this solution length could be chosen. For \OneMax, the best choice is to set $\p = (1/n, \ldots, 1/n)$ as has been proven in~\cite{Sudholt13, Witt13j} (note here that for fixed problem sizes, due to the symmetry of \onemax, non-uniform mutation rates cannot be advantageous over uniform ones). This results in an expected run time of $\Theta(n \log n)$.

 For \LeadingOnes, if the true length $n$ is known, any setting of the bit-flip probabilities leads to an expected run time of $\Omega(n^2)$ regardless of the choice of $\p$, as the next lemma shows. 

\begin{lemma}
\label{LEM:LOLOWERBOUND}
For any fixed solution length $n$ and any vector $\p=(p_1, \ldots, p_n)$ of mutation probabilities, the expected run time of the \OEAp on $\leadingones_n$ is $\Omega(n^2)$.
\end{lemma} 

\begin{proof}
It is easy to see by arguments that are mostly identical to the ones in~\cite[Section 3.3]{BottcherDN10} that the expected run time of the \OEAp on $\leadingones_n$ is 
$$
\sum_{i=1}^n \frac{1}{2p_i}\frac{1}{\prod_{j=1}^{i-1}(1-p_j)}.
$$
Using this bound one can easily show that we can assume without loss of generality that the mutation probabilities are monotonically increasing, i.e., $p_i \leq p_{i+1}$ holds  all $i \in [n]$. Indeed, if for some $k \in [n]$ $p_k > p_{k+1}$ holds, then the expected run time of the \OEAp is larger than that of the \OEAq with $\q=(q_1, \ldots, q_n)$, $q_k=p_{k+1}$, $q_{k+1}=p_k$, and $q_i=p_i$ for $i \notin \{k,k+1\}$.

We now regard the time it takes the \OEAp to produce for the first time a search point of fitness at least $k:=\lfloor n/3 \rfloor$.
Following~\cite{BottcherDN10} this takes in expectation
\begin{equation}
\label{eq:FirstKElementsSum}
\sum_{i=1}^k \frac{1}{2p_i}\frac{1}{\prod_{j=1}^{i-1}(1-p_j)} \geq \sum_{i=1}^k \frac{1}{2p_k} = \Theta(n/p_k).
\end{equation}
fitness evaluations.

Furthermore, we have
$$
\prod_{j=k}^{2k-1}(1-p_j) \leq (1-p_k)^k \leq e^{-p_k k}, 
$$
which shows that the \OEAp spends in expectation 
\begin{equation}
\label{eq:FirstKElementsSum2}
\frac{1}{2p_{2k}}\frac{1}{\prod_{j=1}^{2k-1}(1-p_j)} \geq e^{p_k k}
\end{equation}
iterations on fitness level $2k$.

Equations~\eqref{eq:FirstKElementsSum} and~\eqref{eq:FirstKElementsSum2} prove that the overall expected optimization time of the \OEAp on $\leadingones_n$ is $\Omega(n/p_k + \exp(p_k k/2))$. For all possible choices of $p_k$ this expression is $\Omega(n^2)$ as can be easily seen using a case distinction (for $p_k = O(1/n)$ the first summand is $\Omega(n^2)$, while for $p_k=\omega(1/n)$ the second one is growing at an exponential rate).
\end{proof}
 
Using these lower bounds for fixed solution lengths, Jensen's Inequality and the convexity of $n \mapsto n \log n$ and $n \mapsto n^2$, respectively, we get the following general lower bound.
 
\begin{theorem}
\label{thm:generalLowerBound}
Let $D$ be any distribution on $\N$ with a finite expectation of $m$. 
Then the expected run time of any \OEAp on $\onemax_D$ is $\Omega(m \log m)$ and the expected run time of any \OEAp on $\leadingones_D$ is $\Omega(m^2)$. Both bounds apply also to the setting in which the algorithm designer can choose the mutation probabilities $\p=(p_1, \ldots, p_n)$ \emph{after} the solution length $n~\sim~D$ has been drawn.
\end{theorem}

Using Equation~\eqref{eq:expectedValue}, we get the following corollary.

\begin{corollary}
\label{cor:lowerBoundOnTrunkated}
Let $N \in \N$ and $q \geq 1/N$. 
Let $D= \TrunkGeo(N,q)$ or $D=\Geo(q)$. The expected run time of any \OEAp on $\onemax_D$ is $\Omega(q^{-1} \log q^{-1})$ and the expected run time of any \OEAp on $\leadingones_D$ is $\Omega(q^{-2})$. Both bounds apply also to the setting in which the algorithm designer can choose the mutation probabilities $\p=(p_1, \ldots, p_n)$ \emph{after} the solution length $n~\sim~D$ has been drawn.
\end{corollary}

\subsection{Known Upper Bounds}\label{sec:knownResults}

Cathabard et al.~\cite{PKLFoga11} analyze the run time of the \OEAp with uniform mutation probabilities $p_1=\ldots = p_N = 1/N$ and of the \OEAi with $p_i=1/(i+1)$, $1 \leq i \leq N$.

For \onemax they obtain the following results.  

\begin{theorem}[Results for \onemax from~\cite{PKLFoga11}]
\label{thm:PKLOM}
Let $N \in \N$, $\eps \in (0,1)$, and $q=N^{-\eps}$.
For $D = \TrunkGeo(N,q)$ the expected run time of the 
\OEAp with $\p=(1/N, \ldots, 1/N)$ on 
$\onemax_{D}$ is $\Theta(N \log q^{-1})$, while the expected run time of the \OEAi on $\onemax_{D}$ is $O(q^{-2} \log N)$.
\end{theorem}
This result shows that the \OEAp with $\p=(1/N, \ldots, 1/N)$ outperforms the \OEAi for $q < 1/\sqrt{N}$, while the latter algorithm is preferable for larger $q$. As we shall see in the following section one should not conclude from this result that non-uniform bit flip probabilities are the better choice for this problem.

\textbf{Remark:} By using a slightly more careful analysis than presented in~\cite{PKLFoga11}, the bound for the \OEAi on $\onemax_{D}$ can be improved to $O(q^{-2} \log q^{-1})$. In fact, an analysis similar to the one in Section~\ref{sec:PKLour}, that is disregarding outcomes that are much larger than the expectation, will give that result.
It can also be shown that the requirement $q=N^{-\eps}$ is not needed as the $O(q^{-2} \log q^{-1})$ holds for all $q>1/N$. 
It also holds for the (non-truncated) geometric distribution $D=\Geo(q)$.

For \leadingones, Cathabard et al. show the following results.
\begin{theorem}[Results for \leadingones from~\cite{PKLFoga11}]
\label{thm:PKLLO}
For $N$, $\eps$, $q$, and $D$ as in Theorem~\ref{thm:PKLOM}, the expected run time of the \OEAp with $\p=(1/N, \ldots, 1/N)$ on $\leadingones_{D}$ is 
$\Theta(N q^{-1})$, while the expected run time of the \OEAi on $\leadingones_{D}$ is 
$\Theta(q^{-3})$.
\end{theorem}
Thus also for $\leadingones$ the \OEAi performs better than the \OEAp with $\p=(1/N, \ldots, 1/N)$ when $q>1/\sqrt{N}$ while the uniform \OEAp should be preferred for smaller~$q$. 
 
\textbf{Remark:} As in the $\onemax$ case the $\Theta(q^{-3})$ bound for the \OEAi holds more generally for all geometric distributions $\Geo(q)$ with parameter $q>1/N$. 

From Theorems~\ref{thm:PKLOM} and~\ref{thm:PKLLO} we can see that for both $\onemax_D$ and $\leadingones_D$ the \OEAi looses a factor of $1/q$ with respect to the lower bound given by Corollary~\ref{cor:lowerBoundOnTrunkated}. This will be improved in the following section.

\subsection{Optimal Upper Bounds With Uniform Mutation Probabilities}
\label{sec:PKLour}

We show that for $D$ being the (truncated or non-truncated) geometric distribution there exist bit flip probabilities $\p=(p_1, \ldots, p_N)$ and $\p=(p_i)_{i \in \N}$, respectively, such that the expected run time of the \OEAp on $\onemax_D$ and $\leadingones_D$ is significantly lower than those of the two algorithms studied by Cathabard et al. 
The expected run times of our algorithm match the lower bounds given in Corollary~\ref{cor:lowerBoundOnTrunkated} and are thus optimal in asymptotic terms.

In both cases, i.e., both for $\onemax_D$ and for $\leadingones_D$, the mutation rates yielding the improvement over the results in~\cite{PKLFoga11} are uniform. Our results therefore imply that for these two problems, unlike conjectured in~\cite{PKLFoga11}, one cannot gain more than constant factors from using non-uniform mutation probabilities.  

The key observation determining our choice of the mutation probability is the fact that the (truncated) geometric distribution is highly concentrated. Hence, if we know the parameters of the distribution, we can choose the mutation probability such that it is (almost) reciprocal in each position to the expected length of the solution. Thus, in the setting of~\cite{PKLFoga11}, i.e., for the truncated geometric distribution with parameters $N$ and $q$, we set $p_i:=q/2$ for all $i \in [N]$ (recall equation~\eqref{eq:expectedValue}).
Our approach naturally also works for the (non-truncated) geometric distribution $\Geo(q)$, which is also highly concentrated around its mean $1/q$.

We remark without proof that similar results hold for other distributions that are highly concentrated around the mean, e.g., binomial distributions, and also highly concentrated unbounded distributions, such as Poisson distributions.

\begin{theorem}\label{thm:OurRandomLengthOneMax}
For $N \in \N$ let $1/N \leq q=q(N)<1/2$. 
For $D= \Geo(q)$ and $D= \TrunkGeo(N,q)$ the expected run time of the \OEAp with $\p=(q/2, \ldots, q/2)$ on $\OneMax_{D}$ is $\Theta(q^{-1} \log q^{-1})$.
\end{theorem}

For the proof we will use the following upper bound for the expected run time of the \oea on \OneMax. A similar upper bound can be found in~\cite[Theorem 4.1]{Witt13j}.
\begin{lemma}[{\cite[Theorem 8]{Sudholt13}}]
\label{lem:oeaOM}
For a fixed length $n$ and a uniform mutation vector $\p=(p, \ldots, p)$ with $0<p<1$, the expected run time of the \OEAp on $\OneMax_n$ is at most $(\ln(n) + 1)/(p(1-p)^{n})$. 
\end{lemma}

\begin{proof}[Proof of Theorem~\ref{thm:OurRandomLengthOneMax}]
We first consider $D = \TrunkGeo(N,q)$. We do not worry about constant factors in this analysis and thus bound some expressions generously.

Using Lemma~\ref{lem:oeaOM} we can bound the expected run time of the \OEAp on $\OneMax_{D}$ from above by
\begin{align}
\label{eq:PKLour1}
& \sum_{n=1}^{N-1}{\frac{q(1-q)^{n-1}(\ln(n)+1)}{q/2 (1-q/2)^{n}}} +  \frac{(1-q)^{N-1}(\ln(N)+1)}{q/2 (1-q/2)^{N}}.
\end{align}
To bound the last summand in this expression, we first observe that, for all positive $n$,
\begin{align} 
\label{eq:nett}
(1-\tfrac{q}{2})^{n}
 = 
(1-q+\tfrac{q^2}{4})^{n/2}
>
(1-q)^{n/2}.
\end{align}
This shows that the last summand in~\eqref{eq:PKLour1} is at most
\begin{align*}
2 (1-q)^{N/2-1} (\ln(N)+1)/q,
\end{align*}
which is $O(q^{-1} \log q^{-1})$. This can be seen as follows. 
For $q \geq 2 \ln\ln(N)/N$ it holds (using the inequality $1-q \leq \exp(-q)$) that 
$(1-q)^{N/2-1} \leq \exp(-qN/2) \leq 1/\ln(N)$ and thus 
$2 (1-q)^{N/2-1} (\ln(N)+1)/q = O(1/q)$, while for 
$1/N \leq q \leq 2 \ln\ln(N)/N$ we have (for some suitably chosen constant $C$)
$(1-q)^{N/2} \ln(N) 
\leq 
(1-1/N)^{N/2} \ln(N)
\leq 
C (\ln(N)-\ln(2\ln\ln N))
= 
C \ln(N/(2\ln\ln N))
\leq
C \ln(1/q).$
 
Using again~\eqref{eq:nett} we bound the first part of the sum~\eqref{eq:PKLour1} by
\begin{align*}
& \frac{2}{1-q} \sum_{n=1}^{N-1}{\frac{(1-q)^n (\ln(n)+1)}{(1-q/2)^n}}\\
& \leq 
\frac{2}{1-q} \sum_{n=1}^{N-1}{(\ln(n)+1)(1-q)^{n/2}}\\
& =
2  \sum_{n=1}^{N-1}{(\ln(n)+1)(1-q)^{n/2-1}}.
\end{align*}
To show that this expression is $O(q^{-1} \log q^{-1})$ we split the sum into blocks of length $k:=\lceil 1/q \rceil$ and use again the inequality $1-q \leq \exp(-q)$. This shows that the last expression is at most 
\begin{align*}
& 2 \sum_{j=0}^{\lceil N/k \rceil-1}{ \sum_{\ell=1}^{k}{\exp(-q(\tfrac{jk+\ell}{2}-1)) (\ln(jk+\ell)+1)}}\\
& \leq 
2k \sum_{j=0}^{\lceil N/k \rceil-1}{
	\exp(-\tfrac{1}{k}(\tfrac{jk}{2}-1)) (\ln(j+1)+ \ln(k)+1)}\\
& = 
O(k \ln k),
\end{align*}
where the last equality can be best seen by first considering that 
$\sum_{j=0}^{\lceil N/k \rceil-1}{
	\exp(-\tfrac{1}{k}(\tfrac{jk}{2}-1)) (\ln(k)+1)} = \Theta(\log k)$, while $\sum_{j=0}^{\lceil N/k \rceil-1}{
	\exp(-\tfrac{1}{k}(\tfrac{jk}{2}-1)) (\ln(j+1))}=O(1)$.
Summarizing the computations above we see that~\eqref{eq:PKLour1} is of order at most $q^{-1} \log q^{-1}$.

For $D= \Geo(q)$ the computations are almost identical. By Lemma~\ref{lem:oeaOM} and~\eqref{eq:nett} the expected run time of the \OEAp on $\onemax_D$ is at most
\begin{align*}
& \frac{2}{1-q} \sum_{n=1}^{\infty}{\frac{(1-q)^{n}(\ln(n) +1)}{(1-q/2)^{n}}}\\ 
& \leq 
2 \sum_{n=1}^{\infty}{(1-q)^{n/2-1}(\ln(n) +1)} = 
O(q^{-1} \log q^{-1}), 
\end{align*} 
which can be seen in a similar way as above by splitting the sum into blocks of size $k:=\lceil 1/q \rceil$ and using $1-q \leq \exp(-q)$.
\end{proof}

It is interesting to note that the expected run time increases to between $\Omega(N)$ and $O(N \log N)$ when the mutation probability is chosen to be $\p=(q, \ldots, q)$. 
This can easily be seen as follows. For the upper bound we use Lemma~\ref{lem:oeaOM} (ignoring the ``+1'' terms which are easily seen to play an insignificant role) to obtain that the expected run time of the \OEAp with $\p=(q, \ldots, q)$ on $\OneMax_{\TrunkGeo(N,q)}$ is at most
$
 \sum_{n=1}^{N-1}{\frac{q(1-q)^{n-1}\ln n}{q (1-q)^{n}}} 
+ \frac{(1-q)^{N-1} \ln N}{q (1-q)^N}
 = 
\sum_{n=1}^{N-1}{\frac{\ln n}{1-q}} + O(\log(N)/q)
 = 
\frac{\ln((N-1)!)}{1-q} + O(N \log N)
 = 
O(N \log N)$. 

We can derive a strong lower bound of $\Omega(N \log N)$ in the case of $2^{-N/3} \leq q \leq 1/N$ from the following one for static solution lengths.
\begin{lemma}[{\cite[Theorem 9]{Sudholt13}}, {\cite[Theorem 4.1]{Witt13j}}]
\label{lem:oeaOMlower}
For a fixed length $n$ and a uniform mutation vector $\p=(p, \ldots, p)$, the expected run time of the \OEAp on $\OneMax_n$ is at least 
$(\ln(n)-\ln\ln n-3)/(p(1-p)^{n})$ 
for $2^{-n/3} \leq p \leq 1/n$ and 
at least 
$(\ln(1/(p^2n))-\ln\ln n-3)/(p(1-p)^{n})$ 
for $1/n \leq p \leq 1/(\sqrt{n}\log n)$.
\end{lemma}

Thus, the expected run time of the \OEAp with $\p=(q, \ldots, q)$ and $2^{-N/3} \leq q \leq 1/N$ on $\OneMax_{\TrunkGeo(N,q)}$ is at least
$\sum_{n=1}^{N-1} q(1-q)^{n-1} \frac{ (\ln(n)-\ln\ln n-3)}{q (1-q)^{n}} 
 \geq 
\sum_{n=1}^{N-1}{ \frac{1}{2} \frac{\ln n}{1-q}}
 = 
\frac{1}{2}\frac{\ln((N-1)!)}{1-q}
 = 
\Omega(N \log N). 
$
Similarly we can get a lower bound of $\Omega(N)$ in case of $1/N \leq q \leq 1/(\sqrt{N}\log N)$ by using the lower bound of $1/(q (1-q)^{n})$ for any fixed solution length $n$.

We now turn our attention to the \LeadingOnes problems, where a similar approach as above yields the following result.

\begin{theorem}\label{thm:OurRandomLengthLeadingOnes}
Let $N \in \N$ and $1/N \leq q \leq 1/2$.
For $D = \TrunkGeo(N,q)$ and $D = \Geo(q)$ the expected run time of the \OEAp with $\p=(q/2, \ldots, q/2)$ on $\leadingones_{D}$ is $\Theta(q^{-2})$. 
\end{theorem}

We will derive this result from the following lemma, which was independently proven in {\cite[Theorem 3]{BottcherDN10}}, {\cite[Corollary 2]{Sudholt13}}, and in a slightly weaker form in~{\cite[Theorem 1.2]{Ladret05}}.

\begin{lemma}[{\cite{BottcherDN10}}, {\cite{Sudholt13}}, and~{\cite{Ladret05}}]
\label{lem:oeaLO}
For a fixed length $n$ and a mutation vector $\p=(p, \ldots, p)$ with $0<p<1/2$, the expected run time of the \OEAp on $\leadingones_n$ is exactly
$1/(2p^2) \left( (1-p)^{-n+1}-(1-p)\right)$.
\end{lemma}

\begin{proof}[Proof of Theorem~\ref{thm:OurRandomLengthLeadingOnes}]
We first consider the case that the solution length is sampled from the truncated geometric distribution $\TrunkGeo(N,q)$. Using Lemma~\ref{lem:oeaLO} and~\eqref{eq:nett} (in the third and in the last step) the expected run time of the \OEAp on $\leadingones_{D}$ is  
\begin{align*}
& \sum_{n=1}^{N-1}{q(1-q)^{n-1} \frac{2}{q^2} \left( (1-q/2)^{-n+1}-(1-q/2)\right)}+A\\
& \leq 
\frac{2}{q}\sum_{n=1}^{N-1}{\left(\frac{(1-q)^{n-1}}{(1-q/2)^{n-1}}\right)} + A\\
& \leq 
\frac{2}{q}\sum_{n=0}^{\infty}{(1-q)^{n/2}} + A\\
& = 
\frac{2}{q}\frac{1}{1-(1-q)^{1/2}} + A
 = 
O(q^{-2}) + A, 
\end{align*}
where $A$ is the summand that accounts for the event that the solution length is $N$, i.e., 
\begin{align*}
A 
&= 
(1-q)^{N-1} \frac{1}{2q^2} \left( (1-q)^{-N+1}-(1-q)\right)
= 
O\left(q^{-2}\right).
\end{align*}

Similarly for $D = \Geo(q)$ the expected run time of the \OEAp on $\leadingones_D$ is bounded from above by
\begin{align*}
& \frac{2}{q}\sum_{n=1}^{\infty}{\frac{(1-q)^{n-1}}{(1-q/2)^{n-1}}}
\leq 
\frac{2}{q}\sum_{n=0}^{\infty}{(1-q)^{n/2}}\\
& = 
\frac{2}{q}\frac{1}{1-(1-q)^{1/2}}
 \leq 
\frac{4}{q^2}, 
\end{align*}
where we recall that the last step follows from~\eqref{eq:nett} for $n=1$, which provides $(1-q)^{1/2} \leq 1-q/2$.
\end{proof}

Just as for $\OneMax_{D}$ (with $D = \TrunkGeo(N,q)$) we see that also on $\leadingones_{D}$ the expected run time increases (in this case to $\Theta(N/q)$) when the mutation probability is chosen to be $\p=(q, \ldots, q)$. 
By Lemma~\ref{lem:oeaLO} this run time equals 
$ \sum_{n=1}^{N-1}{q(1-q)^{n-1} \frac{1}{2q^2} \left( (1-q)^{-n+1}-(1-q)\right)}+A
 = 
\frac{1}{2q}\sum_{n=1}^{N-1}{(1-(1-q)^n)+A}
 = 
\frac{1}{2q}\left(N-1-\frac{1-(1-q)^N}{q} + 1\right)+A
 = 
\Theta(N/q)+A,$ 
where $A$ is the summand that accounts for the event that the solution length is $N$, i.e., 
\begin{align*}
A 
&= 
(1-q)^{N-1} \frac{1}{2q^2} \left( (1-q)^{-N+1}-(1-q)\right)
= 
\Theta\left(q^{-2}\right).
\end{align*}


\section{Arbitrary Solution Lengths}
\label{sec:ArbitrarySolutionLength}\label{SEC:ARBITRARYSOLUTIONLENGTH}

In the setting described in Section~\ref{sec:RandomSolutionLength} it is assumed that the algorithm designer has quite a good knowledge about the solution length. Not only does he know an upper bound $N$ on the solution length, but he may also crucially exploit its distribution. 
Indeed, we make quite heavy use in Theorems~\ref{thm:OurRandomLengthOneMax} and~\ref{thm:OurRandomLengthLeadingOnes} of the fact that the (truncated) geometric distribution is highly concentrated around its expected value.
That so much information is available to the algorithm designer can be a questionable assumption in certain applications. 
We therefore regard in this section a more general setting in which \emph{no} a priori information is given about the possible solution length $n$. That is, we regard a setting in which the solution length can be an arbitrary positive integer. In this setting neither do we have any upper bounds on $n$ nor any information about its distribution. 
 
As before, our task is to decide upon on a sequence $(p_i)_{i \in \N}$ of mutation probabilities $0 \leq p_i \leq 1$. An adversary may then choose the solution length $n$ and we run the \OEAp with $\p=(p_1, \ldots, p_n)$. In practical applications, this can be implemented with a (possibly generous) upper bound on the problem size.

We first show that uniform fixed bit flip probabilities necessarily lead to exponential run times (see Section~\ref{sec:ArbitraryLengthUniformBitFlipProb}). 
We then show two ways out of this problem. In 
Section~\ref{sec:ArbitraryLengthNonUniformBitFlipProb} we consider non-uniform bit flip probabilities and in 
Section~\ref{sec:ArbitraryLengthRandomBitFlipProb} we show that we can have an efficient algorithm with uniform bit flip probabilities if we choose the bit flip probability randomly in each iteration.

\subsection{Uniform Bit Flip Probabilities}
\label{sec:ArbitraryLengthUniformBitFlipProb}
\label{SEC:ARBITRARYUNIFORMSOLUTIONLENGTH}

It seems quite intuitive that if nothing is known about the solution length there is not much we can achieve with uniform bit flip probabilities. 
In fact, for any fixed mutation probability $p \in [0,1]$, we just need to choose a large enough solution length $n$ to see that the \OEAp with uniform mutation probability $p$ is very inefficient. 

More precisely, using the following statement (which is a simplified version of~\cite[Theorem 6.5]{Witt13j}) we get the lower bound regarding optimizing \OneMax with uniform bit flip probabilities stated in Theorem~\ref{thm:ourNegOneMax}.

\begin{theorem}[from {\cite{Witt13j}}]
Let $0<\eps<1$ be a constant.
On any linear function, the expected optimization time of the \OEAp with $\p=(p, \ldots, p)$ and $p = O(n^{-2/3-\eps})$ is bounded from below by
\begin{align*}
\left(1-o(1)\right)\frac{1}{p(1-p)^n}
\min\left\{\ln(n), \ln\left(\frac{1}{p^3n^2}\right)\right\}.
\end{align*}
\end{theorem}

\begin{theorem}
\label{thm:ourNegOneMax}\label{THM:OURNEGONEMAX}
Let $p \in [0,1]$ be a constant. Then there exists a positive integer $n_0 \in \N$ such that for all $n \geq n_0$ the expected run time of the \OEAp with $\p=(p, \ldots, p)$ on $\onemax_n$ is $2^{\Omega(n)}$.
\end{theorem}

It is quite intuitive that for large $p$ the expected optimization time of the \OEAp with $\p=(p, \ldots,p)$ is very large also for small problem sizes, as in this case typically too many bits are flipped in each iteration. This has been made precise by Witt, who showed that for $p$, $n$ with $p=\Omega(n^{\eps-1})$, the expected run time of the \OEAp is $2^{\Omega(n^\eps)}$ with probability at least $1-2^{-\Omega(n^\eps)}$~\cite[Theorem 6.3]{Witt13j}.

For \LeadingOnes we get a similar lower bound from Lemma~\ref{lem:oeaLO}.
\begin{theorem}
\label{thm:ourNegLeadingOnes}
Let $p \in (0,1/2)$. Then the expected run time of the \OEAp with $\p=(p, \ldots, p)$ on $\LeadingOnes_n$ is $2^{\Omega(n)}$.
\end{theorem}
\begin{proof}
From Lemma~\ref{lem:oeaLO} we have that the expected run time of the \OEAp is, for $n$ large enough,
\begin{align*}
 					 \frac{1}{2p^2} \left( (1-p)^{-n+1}-(1-p)\right)
  \geq 		 \frac12 \left( e^{pn-p}-1\right)
  = 			 2^{\Omega(n)}.
\end{align*}
\end{proof}



\subsection{Non-Uniform Bit Flip Probabilities}
\label{sec:ArbitraryLengthNonUniformBitFlipProb}

One way to achieve efficient optimization with unknown solution length is by using non-uniform mutation rates, that is, different bit positions have different probabilities associated for being flipped during a mutation operation.

To state our results we need the concept of \emph{summable sequences}. Such sequences will be the basis for the sequence of bit flip probabilities. 
A brief discussion of summable sequences can be found in Section~\ref{sec:summableSequence} in the appendix. In short, a sequence $(p_i)_{i \in \N}$ is \emph{summable} if its series $(\sum_{k=1}^n{p_k})_{n \in \N}$ converges (that is, if it is bounded). 
The advantage of using summable sequences is that the probability of flipping only one single bit is \emph{always} constant, regardless of the total number of bits considered, i.e., regardless of the problem length $n$. 
This is in contrast to the sequence $(1/(i+1))_{i \in \N}$ considered in~\cite{PKLFoga11}, which is not summable, and which has a chance of
$(1/2)\prod_{i=2}^{n}(1-1/(i+1))=1/n$ of flipping only the first bit and a chance of $(1/n)\prod_{i=1}^{n-1}(1-1/(i+1))=1/n^2$ of flipping only the $n$th bit. For this reason the \OEAi is very inefficient for the setting in which the solution length can be arbitrary.

Theorems~\ref{thm:ourReformOneMax} and~\ref{thm:ourReformLeadingOnes} show that not knowing the solution length $n$ does not harm the run time more than by a factor of order $\log^{1+\eps} n$ with respect to the optimal bound when the problem length is known a priori, cf. also Corollary~\ref{cor:nonUniformBitFlipProb} for an explicit sequence yielding this bound. 
In fact, we prove that the additional cost caused by not knowing the solution length in advance is even a bit smaller, cf. the comments after Corollary~\ref{cor:nonUniformBitFlipProb}. 

We start with the theorem regarding \OneMax.
\begin{theorem}
\label{thm:ourReformOneMax}
Let $(p_i)_{i \in \N}$ be a monotonically decreasing summable sequence with $\Sigma:=\sum_{i=1}^{\infty}{p_i}<1$. 
Then, for any $n \in \N$, the expected run time of the \OEAp with $\p=(p_1, \ldots, p_n)$ on $\onemax_n$ is at most $\log n/(p_n(1-\Sigma))= O(\log n/p_n)$. 
\end{theorem}
\begin{proof}
We make use of the multiplicative drift theorem~\cite[Theorem 3]{DoerrJW12} and show that for every $n$ and every search point $x$ with $n-k$ ones, the probability to create in one iteration of the \OEAp with $\p=(p_1, \ldots, p_n)$ a search point $y$ with $\onemax_n(y)>\onemax_n(x)$ is at least of order $k/p_n$. 
This can in fact be seen quite easily by observing that the probability to increase the \onemax-value of $x$ by exactly one is at least
\begin{align*}
k p_n \prod_{j=1}^n{(1-p_j)}
& \geq 
k p_n (1-\sum_{j=1}^n{p_j})
\geq 
k p_n (1-\sum_{j=1}^\infty{p_j})\\
& = 
k p_n (1-\Sigma). 
\end{align*}
From this an upper bound of $\log n/(p_n (1-\Sigma))$ for the run time of the \OEAp follows immediately from the multiplicative drift theorem.
\end{proof}

Next we consider \LeadingOnes. The proof follows along similar lines as the one for \OneMax and uses a fitness level argument instead of multiplicative drift (using additive drift would also be possible).

\begin{theorem}\label{thm:ourReformLeadingOnes}
Let $(p_i)_{i \in \N}$ be a monotonically decreasing summable sequence with $\Sigma:=\sum_{i=1}^{\infty}{p_i}<1$. 
Then, for any $n \in \N$, the expected run time of the \OEAp with $\p=(p_1, \ldots, p_n)$ on  $\leadingones_n$ is at most 
$n/(p_n(1-\Sigma)) = O(n/p_n)$.
\end{theorem}

\begin{proof}
Let $n,k\in\N$ with $k < n$ and let $x \in \{0,1\}^n$ with $\LO(x)=k-1$. The probability to get in one iteration of the \OEAp with $\p=(p_1, \ldots, p_n)$ a search point $y$ with $\LO(y)>\LO(x)$ is at least 
\begin{align*}
p_k \prod_{j=1}^{k-1}(1-p_j)
& \geq
p_k (1-\sum_{j=1}^{k-1}{p_j})
\geq
p_k (1-\Sigma)
\geq 
p_n (1-\Sigma). 
\end{align*}
By a simple fitness level argument (see, e.g., the work by Sudholt~\cite{Sudholt13} for background and examples of this method), the expected run time of the \OEAp on $\leadingones_n$ is thus at most
$n/(p_n(1-\Sigma))$.
\end{proof}

It is well known that for every constant $\eps >0$ the sequence $(1/(i \log^{1+\eps} i))_{i \in \N}$ is summable (this can be proven via Cauchy's condensation test). It is obviously also monotonically decreasing in $i$. Theorems~\ref{thm:ourReformOneMax} and~\ref{thm:ourReformLeadingOnes}, together with the sequence $(p_i)_{i \in \N}:=(1/(2S i \log^{1+\eps} i))_{i \in \N}$ for $S:=\sum_{i=1}^{\infty}{1/(i \log^{1+\eps} i)}$, therefore imply the following corollary.

\begin{corollary}
\label{cor:nonUniformBitFlipProb}
For every positive constant $\eps$ there exists a sequence of mutation probabilities $(p_i)_{i \in \N}$ such that for any $n$ the expected run time of the \OEAp with $\p=(p_1, \ldots, p_n)$ 
on $\onemax_n$ is $O(n \log^{2+\eps} n)$
and is of order $n^2 \log^{1+\eps} n$ for $\leadingones_n$. 
%
\end{corollary}

The bound from Corollary~\ref{cor:nonUniformBitFlipProb} can be improved by regarding the following summable sequences.

For any $r \in \R$ and any $i \in \N_{\geq 2}$ let 
\begin{align*}
\log^{(i)}r:= 
\begin{cases}
\log_2(\log^{(i-1)}r), & \text{ if } \log^{(i-1)}(r)\geq 2;\\
1, & \text{ otherwise;} 
\end{cases}
\end{align*}
where $\log^{(1)}r:=\log_2 r$ if $r\geq 2$ and $\log^{(1)}r:=1$ otherwise.
For every constant $\eps>0$ and all positive integers $s$, $i$ let 
\begin{align}
\label{def:komischSequence}
p^{s,\eps}_{i}:=1/\left(i (\log^{(s)}(i))^{1+\eps} \prod_{j=1}^{s-1}\log^{(j)}(i) \right).
\end{align}
For every $\eps>0$ and every $s \geq 1$ the sequence $(p^{s,\eps}_{i})_{i \in \N}$ is summable. Furthermore, this sequence clearly is monotonically decreasing. Choosing larger and larger $s$ therefore gives better and better asymptotic run time bounds in Theorems~\ref{thm:ourReformOneMax} and~\ref{thm:ourReformLeadingOnes}.


\subsection{Randomized Bit Flip Probability}
\label{sec:ArbitraryLengthRandomBitFlipProb}

In the conclusions of~\cite{PKLFoga11} the authors ask the following: how can we optimize efficiently when an upper bound $N$ on the problem length is known, but only $n$ bits \emph{at unknown positions} are relevant for the fitness? It is not difficult to see that our previous solutions with non-uniform bit flip probabilities will not be able to assign appropriate bit flip probabilities to the \emph{relevant} bit positions. 
However, any uniform choice of bit flip probabilities will effectively ignore irrelevant bit positions. 
In this section we consider a variation of the \oea where the bit flip probability $p$ is chosen randomly from a distribution $Q$ on $(0,1)$ in each iteration (the distribution $Q$ does not change over time). This mutation probability is then applied independently to each bit, i.e., each bit of the current best solution is independently flipped with probability $p$. See Algorithm~\ref{alg:OEAD} for the detailed description of the \OEAQ.

\begin{algorithm2e}%
	\textbf{Initialization:} 
	Sample $x \in \{0,1\}^n$ uniformly at random and query $f(x)$\;
 \textbf{Optimization:}
	\For{$t=1,2,3,\ldots$}{
		Sample bit flip probability $p_t$ from $Q$\;
		\For{$i=1, \ldots, n$}
			{\label{line:mut2}With probability $p_t$ set $y_i\assign 1-x_i$ and set $y_i \assign x_i$ otherwise\;}
		Query $f(y)$\;
		\lIf{$f(y)\geq f(x)$}{$x \assign y$\;}	
}
\caption{The \OEAQ for a distribution $Q$ on $(0,1)$ optimizing a pseudo-Boolean function $f:\{0,1\}^n \rightarrow \R$.}
\label{alg:OEAD}
\end{algorithm2e}

To make the problem more explicit, we are asked to find a distribution $Q$ on $[0,1]$ such that the \OEAQ efficiently optimizes for any $n \in \N$ and any pairwise different $b_1, \ldots, b_n \in \N$ the functions  
\begin{align*} 
& \OneMax_{b_1, \ldots, b_n}(x) :=\sum_{i=1}^{n}{x_{b_i}} \text{, respectively} \\ 
& \LeadingOnes_{b_1, \ldots, b_n}(x):= 
	\max\{i \in [0..n] \mid \forall j \leq i: x_{b_j}=1\}.
\end{align*}
In Theorems~\ref{thm:RandomBitFlipProbabilityOneMax} and~\ref{thm:RandomBitFlipProbabilityLeadingOnes} we show that such a distribution $Q$ exist. That is, there is a distribution $Q$ such that the corresponding \OEAQ efficiently optimizes any $\OneMax_{b_1, \ldots, b_n}$ and any $\LeadingOnes_{b_1, \ldots, b_n}$ function, regardless of the number of relevant bits and regardless of their positions.

We start with our main result regarding \OneMax. 
\begin{theorem}
\label{thm:RandomBitFlipProbabilityOneMax}
Let $(p_i)_{i \in \N} \in (0,1)^\N$ be a monotonically decreasing summable sequence. Set $\Sigma:=\sum_{j=1}^{\infty}{p_j}$. Let $Q$ be the distribution which assigns the mutation probability $1/i$ a probability of $p_i/\Sigma$. 

For any $n \in \N$ and any pairwise different positive integers $b_1, \ldots, b_n$ the expected run time of the \OEAQ on 
$\OneMax_{b_1, \ldots, b_n}$ is $O\left(\log(n) / p_{2n}\right)$. 
\end{theorem}

\begin{proof}
The probability to sample a mutation probability between $1/(2n)$ and $1/n$ is 
\begin{align*}
\sum_{j=n}^{2n}{p_j}  
  \geq np_{2n}.
\end{align*} 
We disregard all iterations in which we do not sample a mutation probability between $1/(2n)$ and $n$ (they can only be beneficial). Thus, on average, we consider at least one iteration out of  $1/(np_{2n})$.

Assuming that $x$ is a search point with $n-\ell$ ones (in the relevant positions) and that the sampled bit flip probability $p$ satisfies $1/(2n) \leq p \leq 1/n$, the probability to make a progress of exactly one is at least 
\begin{align*}
\ell p(1-p)^{n-1} 
\geq 
\ell/(2n) (1-1/n)^{n-1}
\geq 
\ell/(2en).
\end{align*}
Thus, we have an expected progress in each iteration of at least
$$
\frac{\ell}{2en} np_{2n} = O\left(\ell p_{2n}\right).
$$
Therefore, by the multiplicative drift theorem~\cite[Theorem 3]{DoerrJW12}, we
need in expectation $O(\log (n) / p_{2n})$ iterations to optimize function $\OneMax_{b_1, \ldots, b_n}$. 
\end{proof}

For $\leadingones$ we obtain the following.
\begin{theorem}
\label{thm:RandomBitFlipProbabilityLeadingOnes}
Let $(p_i)_{i \in \N}$ and $Q$ as in Theorem~\ref{thm:RandomBitFlipProbabilityOneMax}.
 
For any $n \in \N$ and any pairwise different $b_1, \ldots, b_n \in \N$  the expected run time of the \OEAQ on $\LeadingOnes_{b_1, \ldots, b_n}$ is $O\left(n / p_{2n}\right)$. 
\end{theorem}

\begin{proof}
This proof follows along similar lines as the one for \OneMax. We have again that the probability to have a bit flip probability between $1/(2n)$ and $1/n$ in an iteration is at least $np_{2n}$.

Let $x$ be a search point with $\LeadingOnes_{b_1, \ldots, b_n}(x) = \ell$. Given a mutation probability $p$ between $1/(2n)$ and $1/n$, the probability to create in one iteration of the \OEAQ a search point $y$ of fitness greater than $\ell$ is at least
\begin{align*}
p(1-p)^{\ell-1} 
\geq 
1/(2n) (1-1/n)^{n-1}
\geq 
1/(2en).
\end{align*}
Thus, we have an expected progress in each iteration of at least
$$
\frac{1}{2en} np_{2n} = O(p_{2n}).
$$
Therefore, by the fitness level method (see again~\cite{Sudholt13} for a discussion of this method), we
need in expectation $O(n / p_{2n})$ iterations to optimize $\LeadingOnes_{b_1, \ldots, b_n}$. 
\end{proof}

By choosing the summable sequence with entries as in (\ref{def:komischSequence}) and $s=1$, the two theorems above immediately yield the following result.

\begin{corollary}
\label{cor:RandomBitFlipProbability}
The expected run time of the described \OEAQ with $Q$ using the summable sequence (\ref{def:komischSequence}) with $s=1$ on $\OneMax_{b_1, \ldots, b_n}$ is
$O\left(n \log^{2+\eps}n\right)$
and on $\LeadingOnes_{b_1, \ldots, b_n}$ it is
$O(n^2 \log^{1+\eps}n).$
\end{corollary} 
Note that, just as discussed after Corollary~\ref{cor:nonUniformBitFlipProb}, choosing larger and larger~$s$ gives asymptotically better and better bounds.

\section{Summary and Outlook}
We have analyzed the performance of variants of the \oea in the presence of unknown solution lengths. While for highly concentrated solution length non-uniform mutation probabilities are not advantageous (or at least not to a significant degree), they are crucial in a setting in which we do not have any knowledge about the solution length. Surprisingly, even in the latter situation, a sequence of (non-uniform) mutation probabilities exists such that the corresponding \oea is almost optimal, simultaneously for all possible solution lengths. 

We have also investigated a setting in which the relevant bit positions can be arbitrary in number and position. Possibly even more surprisingly, even this can be handled quite efficiently by a \oea variant for the two test functions \onemax and \leadingones. 

We believe the setting of unknown solution length to be relevant for numerous real-world applications. 
As a next step toward a better understanding of how this uncertainty can be tackled efficiently with evolutionary algorithms, we suggest to investigate more challenging function classes, e.g., starting with the class of all linear functions. It is not clear a priori if bounds similar to the ones presented in Section~\ref{SEC:ARBITRARYSOLUTIONLENGTH} can be achieved for such problems. 

From a mathematical point of view it would also interesting to investigate the tightness of our bounds in Section~\ref{SEC:ARBITRARYSOLUTIONLENGTH}. We do not know whether some choice of mutation probabilities gives an upper bound of $O(n \log n)$ for \onemax or $O(n^2)$ for \leadingones. We recall that the sequences $\left(1/(n \log(n))\right)_{n \in \N}$ as well as $(1/p_i^{\infty,\eps})_{i \in \N}$ with $p_i^{\infty,\eps}:=\lim_{s \rightarrow \infty}p_i^{s,\eps}$ are not summable. Removing the gap entirely is therefore likely to require a substantially different approach.

\subsection*{Acknowledgments}
This research benefited from the support of the ``FMJH
Program Gaspard Monge in optimization and operation research'',
and from the support to this program from EDF.

Parts of this work have been done while Timo K\"otzing was visiting the Universit\'e Pierre et Marie Curie (Paris 6). 


\appendix 
\section{Summable Sequences}\label{sec:summableSequence}

For a sequence $\p=(p_i)_{i \in \N}$ the $k$-th term of its associated \emph{series} is the partial sum $\Sigma_k(\p)=\sum_{i=1}^k{p_i}$. The sequence $\p$ is said to be \emph{summable} if its associated series converges, i.e., if $\lim_{k \rightarrow \infty} \Sigma_k$ exists. For $\p \in \R_{\geq 0}$ this is the case if and only if the sequence $(\Sigma_k)_{k \in \N}$ (note that the series forms a sequence itself) is bounded. The limit $\lim_{k \rightarrow \infty} \Sigma_k$ is often abbreviated by $\sum_{i=1}^{\infty}{p_i}$, a notation that we adopt here as well.

It is well known that the sequence $(1/n^2)_{n\in \N}$ is summable. Similarly, for all $\eps>0$ the sequence $(1/n^{1+\eps})_{n\in \N}$ is summable, while the harmonic sequence $(1/n)_{n\in \N}$ is not. Note that the latter is the sequence of non-uniform bit flip probabilities used in the work of Cathabard et al.~\cite{PKLFoga11}. 

For our purposes in Section~\ref{sec:ArbitrarySolutionLength} we need summable sequences that are as large as possible (with respect to $O$-notation). As the examples above show, these sequences have to be in between $(1/n^{1+\eps})_{n\in \N}$ and $(1/n)_{n\in \N}$. The sequences defined after Corollary~\ref{cor:nonUniformBitFlipProb} are already pretty large. Note that for $s \rightarrow \infty$ these sequences converge to the sequence with entries
	$$
	p_n := 1/\left(n \prod_{j=1}^{\infty} \log^{(j)}(n) \right)
	$$
This sequence	is well-defined (since, for each $n$, almost all terms in the product are $1$), but it is \emph{not} summable.
For the sake of completeness we note that there are summable sequences which are larger than any sequence $(p_n^{s,\eps})_{n \in \N}$, but a further discussion is beyond the scope of this paper.
}

\end{document}